\RequirePackage{xcolor}
\documentclass{article}

\usepackage{microtype}
\usepackage{graphicx}
\usepackage{subfigure}
\usepackage{enumitem}
\usepackage{booktabs} 

\usepackage{hyperref}

\usepackage[accepted]{icml2019-arxiv}

\icmltitlerunning{Cross-Entropy Loss and Low-Rank Features Have Responsibility for Adversarial Examples}

\usepackage{graphicx}
\usepackage{float}
\usepackage{amssymb}
\usepackage{amsmath}
\usepackage{dsfont}
\usepackage{centernot}
\usepackage{euscript}
\usepackage{diagbox}
\usepackage{mdframed}
\usepackage{todonotes}

\usepackage{amsthm}

\usepackage{lipsum}
\newtheorem{theo}{Theorem}
\newtheorem{lemma}{Lemma}
\newtheorem{corol}{Corollary}
\newtheorem{prop}{Proposition}

\theoremstyle{definition}
\newtheorem{remark}{Remark}

\begin{document}

\twocolumn[
\icmltitle{Cross-Entropy Loss and Low-Rank Features Have Responsibility\\ for Adversarial Examples}



\icmlsetsymbol{equal}{*}

\begin{icmlauthorlist}
\icmlauthor{Kamil Nar\ \ \ }{ucb}
\icmlauthor{Orhan Ocal\ \ \ }{ucb}
\icmlauthor{S.~Shankar Sastry\ \ \ }{ucb}
\icmlauthor{Kannan Ramchandran}{ucb}
\end{icmlauthorlist}

\icmlaffiliation{ucb}{\noindent Authors are with Department of Electrical Engineering and Computer Sciences, University of California, Berkeley\vspace{0.01in}}

\icmlcorrespondingauthor{Kamil Nar}{ \includegraphics[scale=0.4]{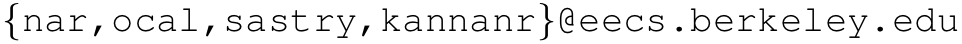}\vspace{-0.17in}}


\vskip 0.3in
]



\printAffiliationsAndNotice 

\begin{abstract}
State-of-the-art neural networks are vulnerable to adversarial examples; they can easily misclassify inputs that are imperceptibly different than their training and test data.
In this work, we establish that the use of cross-entropy loss function and the low-rank features of the training data have responsibility for the existence of these inputs.
Based on this observation, we suggest that addressing adversarial examples requires rethinking the use of cross-entropy loss function and looking for an alternative that is more suited for minimization with low-rank features.
In this direction, we present a training scheme called differential training, which uses a loss function defined on the differences between the features of points from opposite classes.
We show that differential training can ensure a large margin between the decision boundary of the neural network and the points in the training dataset. This larger margin increases the amount of perturbation needed to flip the prediction of the classifier and makes it harder to find an adversarial example with small perturbations. 
We test differential training on a binary classification task with CIFAR-10 dataset and demonstrate that it radically reduces the ratio of images for which an adversarial example could be found -- not only in the training dataset, but in the test dataset as~well. 

\end{abstract}

\section{Introduction}

Despite their high accuracy on training and test datasets, state-of-the-art neural networks are vulnerable to adversarial examples: they can easily misclassify inputs that are indistinguishable from the training and test data and express very high confidence for their wrong predictions~\cite{Szegedy}. 
Several methods have recently been introduced to generate these adversarial inputs~\cite{Goodfellow2015Adversarial, Carlini2017towards, Universal,pmlr-v80-athalye18a}; and simplicity and effectiveness of these methods have reinforced the concerns about the use of neural networks in many tasks.

The presence of adversarial examples was initially attributed to the high nonlinearity of deep neural networks~\cite{Szegedy}.
Later, however, it was shown that a network with few layers and a high dimensional input space could also suffer from this problem~\cite{Goodfellow2015Adversarial}.
Support vector machines with radial basis function, on the other hand, were robust to these malicious inputs: their accuracy on test datasets and adversarial examples were comparable.
Based on these observations, it was claimed that neural networks, unlike support vector machines, failed to introduce adequate nonlinearity as a feature mapping, and this was suggested to be the main explanation for the existence of adversarial examples~\cite{Goodfellow2015Adversarial}.

It is correct that neural networks and support vector machines differ in their level of nonlinearity and their level of robustness against adversarial examples, but this fact on its own does not suffice to build a causal relation between the adversarial examples and the nonlinearity of the classifier.
There are many other aspects that neural networks and support vector machines differ in and any of these factors may also have responsibility for the presence of adversarial examples.
A major one of these factors is the training~procedure.

Training a support vector machine involves solving a convex optimization problem defined with the hinge loss function~\cite{hastie_09_elements-of.statistical-learning}.
Due to convexity of the problem, the choice of optimization algorithm has no influence on the classifier obtained at the end of training.
In contrast, training a neural network requires solving a nonconvex problem, and the dynamics of the optimization algorithm becomes critical for the solution.
It determines the local optimum obtained, and hence, the decision boundary of the trained network.

The existence of adversarial examples is the manifestation of a poor margin between the decision boundary of the network and the points in the training and test datasets~\cite{Pascal}.
What is interesting is the closeness of the training points to the decision boundary: for some reason, the decision boundary resides extremely close to the training points even after the training is complete -- although the main purpose of training is to find a boundary that is reasonably far away from these points.
We seek out a reason for this poor margin among the ingredients of neural network training that are widely taken for granted: the gradient methods and the cross-entropy loss function.

\subsection{Our contributions}
\vspace{0.1in}
\begin{enumerate}[topsep=1ex,itemsep=2ex,partopsep=2ex,parsep=1ex]
\item We show that if a linear classifier is trained by minimizing the cross-entropy loss function via the gradient descent algorithm, and if the features of the training points lie on a low-dimensional affine subspace, then the margin between the decision boundary of the classifier and the training points could become much smaller than the optimal value.

\item We show that the penultimate layer of neural networks are very likely to produce low-rank features, and we provide empirical evidence for this on a binary classification task with CIFAR-10 dataset. Combined with the first contribution, this suggests that neural networks could have a poor margin in their penultimate layer, and consequently, very small perturbations in this layer can easily flip the decision of the classifier.

\item In order to improve the margin, we put forward a training scheme called \emph{differential training}, which uses a loss function defined on the differences between the features of the points from opposite classes. We show that this training scheme allows finding the solution with the largest hard margin for linear classifiers while still using the gradient descent algorithm.

\item We introduce a loss function that improves the margin for nonlinear classifiers and display its effectiveness on a synthetic problem. Then we test this loss function on a binary classification task with CIFAR-10 dataset, and show that it prevents the Projected Gradient Descent Attack~\cite{Madry, Kurakin2016} from being able to find an adversarial example for most of the training and test data.

\item On CIFAR-10 dataset, we empirically show that the network produced by differential training generalizes well over the adversarial examples.
That is, the accuracy of the network is virtually the same on adversarial examples generated from the training dataset and on those generated from the test dataset.
This result is critical given that the networks trained with robust optimization were shown not to generalize on adversarial examples~\cite{schmidt2018}. 
\end{enumerate}

\subsection{Related Works}
\label{sec:related-works}
The minimization of cross-entropy loss function via the gradient descent algorithm has recently been studied for linear classifiers, and its solution has been shown to be equivalent to a support vector machine~\cite{Soudry-March-2018}.
However, it has not been emphasized that the separating hyperplane produced by the cross-entropy minimization is constrained to pass through the origin in an augmented space.
We show that this fact could cause the margin of the classifier to be drastically small if the features of the dataset lie in a low-dimensional affine subspace in a high dimensional feature space.
We also show that this case is not atypical when a neural network is trained with the gradient descent algorithm, and we build a connection between this fact and the existence of adversarial examples.

It is known that if a support vector machine is formulated to find a separating hyperplane passing through the origin, the decision boundary of the classifier will be smaller than the optimal value.
In order to overcome this problem and to speed up online learning algorithms {\bf for support vector machines}, the idea of using the differences between the points from opposite classes has previously been suggested in~\cite{Ishibashi, Keerthi99}.
We show that a similar idea in differential training also improves the margin \textbf{when a neural network is being trained with a gradient-based method}.

Differential training uses the differences between the features of the training points from opposite classes. This training scheme has been intentionally introduced to improve the dynamics of the gradient descent algorithm on the training cost function; and we consider it as using an alternative cost function in the sequel since the choice of cost function is very critical.
However, the procedure could also be considered as using an identical pair of networks in the network architecture, which is closely related to the Siamese Networks~\cite{Bromley93,Chopra05}. These networks were previously shown to perform well if limited data were available from any of the classes in a classification task~\cite{SiameseOneShot}. Our work shows that this architecture can also provide a large margin between the decision boundary of the classifier and the training points, and consequently, be more robust to adversarial examples \textbf{if} the network is trained with the cost function we suggest in Section \ref{subsec:diff-nonlinear}.

\section{Cross-Entropy Loss on Low-Rank Features Leads to Poor Margins}

Cross-entropy loss function is almost the sole choice for classification tasks in practice. 
Its prevalent use is backed theoretically by its association with the minimization of the Kullback-Leibler divergence between the empirical distribution of a dataset and the confidence of the classifier for that dataset. 
Given the particular success of neural networks for classification tasks~\cite{AlexNet,VGG,Resnet}, there seems to be little motivation to search for alternatives for this loss function, and most of the software developed for neural networks incorporates an efficient implementation for it, thereby facilitating its further use.

Nevertheless, there seems to be a \textbf{typical} case where the use of cross-entropy loss function can create a problem for the classifier, as shown in Figure~\ref{fig:margin-figure}. The source of this problem is pointed out in Theorem~\ref{theo:exact-subspace}. 

\begin{figure}
\centering
\includegraphics[width=0.8\linewidth]{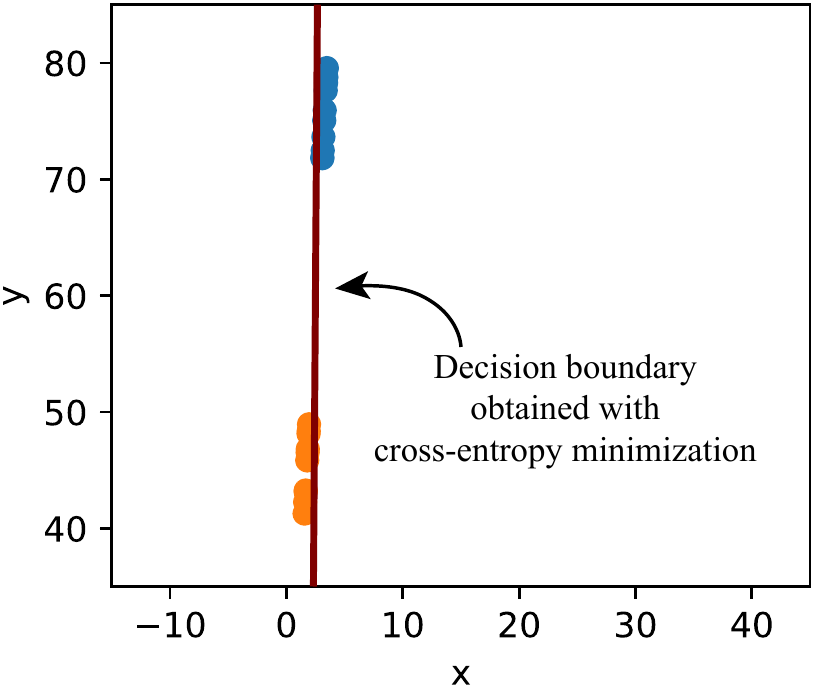}
\caption{Orange and blue points lie on a low-dimensional affine subspace in $\mathbb R^2$, and they represent the data from two different classes. Cross-entropy minimization for a linear classifier on these points leads to the decision boundary shown with the solid line, which attains an extremely poor margin.
}
\label{fig:margin-figure}
\end{figure}

\begin{theo} Assume that the points $\{x_i\}_{i\in I}$ and $\{y_j\}_{j\in J}$ are linearly separable and  lie in an affine subspace; that is, there exist a set of orthonormal vectors $\{r_k\}_{k \in K}$ and a set of scalars $\{\Delta_k\}_{k \in K}$ such that
\[ \langle r_k,  x_i \rangle = \langle r_k, y_j \rangle = \Delta_k \quad \forall i \in I,\ \forall j \in J,\ \forall k \in K. \]
Let $\langle \overline w, \cdot \rangle + B = 0$ denote the decision boundary obtained by minimizing the cross-entropy loss function
\begin{equation}
-\sum_{i \in I} \log\left( {e^{w^\top x_i + b} \over 1 + e^{w^\top x_i + b}} \right) 
-\sum_{j \in J} \log\left( {1 \over 1 + e^{w^\top y_j+b}} \right), \nonumber
\end{equation}
and assume that $\overline w$ and $B$ are scaled such that
\[ \min_{i \in I, j \in J} \ \langle \overline w, x_i \rangle - \langle \overline w, y_j \rangle = 2.
\]
Then the minimization of the cross-entropy loss yields a margin smaller than or equal to
\[ {1 \over \sqrt{ {1 \over {\gamma}^{2}} + B^2 \sum\nolimits_{k\in K}  \Delta_k^2  }} \]
where $\gamma$ denotes the optimal hard margin given by the SVM solution.
\label{theo:exact-subspace}
\end{theo}

\begin{remark}
Theorem~\ref{theo:exact-subspace} shows that if the training points lie on an affine subspace, and if the cross-entropy loss is minimized with the gradient descent algorithm, then the margin of the classifier will be smaller than the optimal margin value. As the dimension of this affine subspace decreases, the cardinality of the set $K$ increases and the term $\sum_{k \in K} \Delta_k^2$ could become much larger than ${1/\gamma^2}$.
Therefore, as the dimension of the subspace containing the training points gets smaller compared to the dimension of the input space, cross-entropy minimization with a gradient method becomes more likely to yield a poor margin.
\label{remark:subspace}
\end{remark}

The next corollary relaxes the condition of Theorem~\ref{theo:exact-subspace} and allows the training points to be near an affine subspace instead of being exactly on it.

\begin{corol} Assume that the points $\{ x_i\}_{i \in I}$ and $\{y_j\}_{j \in J}$ in $\mathbb R^d$ are linearly separable and there exist a set of orthonormal vectors $\{r_k\}_{k \in K}$ and a set of scalars $\{\Delta_k\}_{k \in K}$ such that
\[ \langle r_k,  x_i \rangle \ge \Delta_k, \ \langle r_k, y_j \rangle \le \Delta_k \quad \forall i \in I, \ \forall j \in J, \ \forall k\in K.\]
Let $\langle \overline w, \cdot \rangle + B = 0$ denote the decision boundary obtained by minimizing the cross-entropy loss, as in Theorem 1. 
Then the minimization of the cross-entropy loss  yields  a margin smaller than or equal to
\[ {1 \over \sqrt{ B^2 \sum\nolimits_{k\in K}  \Delta_k^2  }} \]
\end{corol}
Note that the ability to compare the margin obtained by cross-entropy minimization with the optimal value is lost.
Nevertheless, it highlights the fact that same set of points could be assigned a substantially different margin by cross-entropy minimization if all of them are shifted away from the origin by the same amount in the same direction.

\section{Penultimate Layers of Neural Networks Contain Low-Rank Features}

The results in the previous section were for linear classifiers, and correspondingly, the features of the training points were the points themselves. In this section, we consider neural networks and regard the outputs of their penultimate layer as the features of the training points. Following theorem shows that these features can have a very low rank if the network is trained with a gradient method.

\begin{prop} Given a set of points $\{x_i\}_{i \in I}$, assume that an $L$-layer network is trained by minimizing the cross-entropy loss function:
\[ \min_{w, \theta}\  \sum\nolimits_{i \in I} - \log\left( {e^{w^\top \phi_\theta (x_i)} \over 1 + e^{w^\top \phi_\theta (x_i)}} \right) \]
where $\phi_\theta(x_i)$ is the output of the penultimate layer of the network and represents the features for point $x_i$. Assume that $\phi_\theta$ ends with a linear layer, i.e., 
\[ \phi_\theta(\cdot) = W \cdot h_\theta(\cdot) \]
where $W$ is a matrix and $h_\theta(\cdot)$ is the first $L-2$ layers of the network. If the gradient descent algorithm is initialized with $W[0] = 0$, then the rank of the set $\{ \phi_{\hat \theta} (x_i) \}_{i \in I}$ is at most 1 whenever the algorithm is terminated.
\end{prop}

The assumption on the initialization of the matrix $W$ could be removed if the network has a certain structure -- for example, if the last layer of $h_\theta(\cdot)$ ends with a squishing function such as $\arctan$ or $\tanh$. In this case, the points in  $\{\phi_\theta(x_i)\}_{i\in I}$ keep growing in the same direction if the algorithm is run for long enough, and consequently, this set converges to a set with rank 1 as well. More detail on this case is provided in Appendix \ref{appendix:nonzero-init}.

Note that the only strong assumption in Proposition 1 is the requirement that $\phi_\theta$ ends with a linear layer. Otherwise, $\phi_\theta$ is allowed to contain any type of nonlinear activation functions and convolutional layers.


To empirically verify whether the features in a neural network are still low-rank even when the penultimate layer is nonlinear, we trained a standard network with ReLU activations for a binary classification task on CIFAR-10 dataset. The cross-entropy loss function was minimized with three different optimization schemes to train the network. 
Even though all parameters of the network were initialized as in~\cite{he2015delving},
the features in the penultimate layer had rank 2 if the training cost was minimized via the gradient method with momentum. When the optimization algorithm was changed to Adam or when batch normalization was used during training, the rank of the features still remained much lower than the dimension of the feature~space, as shown in Figure~2.

\begin{figure}[t]
\centering
\includegraphics[width=0.8\linewidth]{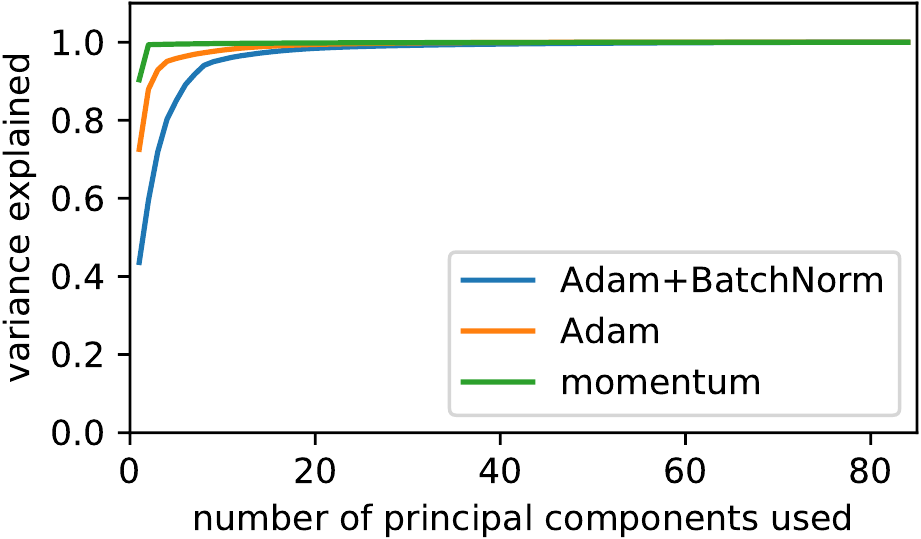}
\caption{The outputs of the penultimate layer of a neural network can be considered as the features of the training points. A four-layer convolutional network is trained by minimizing the cross-entropy loss function via three different optimization schemes. The plot shows the cumulative variance explained for these features as a function of the number of principle components used.
The features lie in a two-dimensional subspace if the gradient method with momentum is used. For the other two algorithms, almost all the variance in the features is captured by the first $20$ principle components out of $84$.}
\end{figure}

\begin{remark} Proposition 1, along with the empirical observations on CIFAR-10 dataset, shows that the low-rankness of the features of the training dataset is not an exceptional case; on the contrary, it can arise in most cases. This is recently supported by~\cite{mahoney2018} as well. 
\end{remark}

Along with the main result of Section 2, the fact that penultimate layer of the network contains low-rank features indicates a small margin between the decision boundary of the classifier and the features in this layer. In other words, small perturbations in the penultimate layer can easily flip the decision of the classifier.



\section{Differential Training Improves Margin}
In previous sections, we saw that the combination of cross-entropy loss function, low-rank features of training dataset, and gradient descent algorithm could lead to a poor margin.
We change the training cost function in the following subsections in order to increase the margin of the classifier.

\subsection{Differential Training for Linear Classifiers}

Consider the binary classification problem with only two training points, $x$ and $y$, from two different classes. If we use cross-entropy loss function to find a linear classifier by minimizing
\[ - \log\left( {e^{w^\top x + b} \over 1 + e^{w^\top x + b}} \right) - \log \left( {1 \over 1 + e^{w\top y + b}} \right),\]
the gradient descent algorithm gives the update rule:
\begin{equation}
w \leftarrow w + \eta \left( x  {e^{-w^\top x - b} \over 1 + e^{-w^\top x - b}} 
-y  {e^{w^\top y + b} \over 1 + e^{w^\top y + b}}  \right)
\label{eqn:weight_update} \end{equation}
where $\eta$ is the learning rate of the algorithm. The update rule for $w$ reveals a critical fact:
even though the optimal direction for $w$ is $x-y$, the increments in $w$ are usually not in this direction.

Now consider the problem of finding a separating hyperplane for a linearly separable dataset.
If the dataset is low rank, the differences between the training points span a low-dimensional subspace. However, at each iteration of the gradient descent algorithm, the increments on the normal vector of the decision boundary will usually contain components outside of this subspace, as can be seen in (\ref{eqn:weight_update}). These increments could be forced to lie in the same subspace by feeding the differences of the points from opposite classes -- instead of the points themselves -- into the loss function. In fact, a loss function of this form enables finding the separating hyperplane with the largest margin with the gradient descent algorithm.

\begin{theo} Given two sets of points $\{ x_i\}_{i\in I}$ and $\{y_j\}_{j \in J}$ that are linearly separable in $\mathbb R^d$, if we solve
\begin{equation} \min_{w \in \mathbb R^d} \ \sum\nolimits_{i \in I} \sum\nolimits_{j \in J} \log(1 + e^{-w^\top ( x_i -  y_j)} ) \label{brand_new_cost} \end{equation}
by using the gradient descent algorithm with a sufficiently small learning rate, then the direction of $w$ converges to the direction of the maximum-margin solution, i.e.
\begin{equation}
 \lim_{t \to \infty} {w(t) \over \|w(t)\|} = {w_\text{SVM} \over \|w_\text{SVM}\|}, \label{thm6_eqn} \end{equation}
where $w_\text{SVM}$ is the solution to the hard-margin SVM problem.
\end{theo}

Minimization of the cost function (\ref{brand_new_cost}) provides only the weight parameter $\hat w$ of the decision boundary. The bias parameter, $b$, could be chosen by plotting the histogram of the inner products $\{\langle \hat w, x_i \rangle \}_{i \in I}$ and $\{ \langle \hat w, y_j \rangle \}_{j \in J}$ and fixing a value for $\hat b$ such that
\begin{subequations}
\begin{align}
\langle \hat w, x_i \rangle + \hat b \ge 0 & \quad \forall i \in I, \label {eqmel1}\\
\langle \hat w, y_j \rangle + \hat b \le 0 & \quad \forall j \in J. \label{eqmel2}
\end{align}
\end{subequations}
The largest hard margin is achieved by 
\begin{equation} \hat b = -{1 \over 2} \min_{i \in I} \langle \hat w, x_i \rangle  - {1 \over 2} \max_{j \in J} \langle \hat w, y_j \rangle.\label{bias_choice} \end{equation}
However, by choosing a larger or smaller value for $\hat b$, it is possible to make a tradeoff between the Type-I and Type-II errors.

The cost function (\ref{brand_new_cost}) includes a loss defined on every pair of data points from the two classes. There are two aspects of this fact:
\begin{enumerate}
\item When standard loss functions are used for classification tasks, we need to oversample or undersample either of the classes if the training dataset contains different  number of points from different classes. This problem does not arise when we use the cost function (\ref{brand_new_cost}). 
\item The number of pairs, $|I|\times |J|$, will usually be much larger than the size of the original dataset, which contains $|I| + |J|$ points. Therefore, the minimization of  (\ref{brand_new_cost}) might appear more expensive than the minimization of the standard cross-entropy loss computationally. However, if the points in different classes are well separated and the stochastic gradient method is used to minimize (\ref{brand_new_cost}), the algorithm could achieve zero  training error after using only a few pairs, which is formalized in Theorem \ref{theo:well-separate}. Further computation is needed only to improve the margin of the classifier. In addition, in our experiments to train a neural network to classify two classes from the CIFAR-10 dataset, only a few percent of $|I| \times |J|$ pairs were observed to be sufficient to reach an accuracy on the test dataset that is comparable to the accuracy of the cross-entropy loss minimization.
\end{enumerate}

\begin{theo} 
\label{theo:well-separate} Given two sets of points $\{x_i\}_{i \in I}$ and $\{y_j\}_{j \in J}$ that are linearly separable in $\mathbb R^d$, assume the cost function (\ref{brand_new_cost}) is minimized with the stochastic gradient method. Define \begin{gather*} R_x = \max\{ \|x_i - x_{i'}\| : i, i' \in I\},\\ R_y = \max\{ \|y_j - y_{j'} \| : j, j' \in J\},\end{gather*} and let $\gamma$ denote the hard margin that would be obtained with the SVM: 
\[ 2\gamma = \max\nolimits_{u \in \mathbb R^d} \min\nolimits_{i \in I, j \in J}  \ \langle x_i -y_j , {u / \|u\|} \rangle. \]
If $2\gamma \ge {5} \max (R_x, R_y)$, then the stochastic gradient algorithm produces a weight parameter, $\hat w$, only in one iteration which satisfies the inequalities (\ref{eqmel1})-(\ref{eqmel2}) along with the bias, $\hat b$, given by (\ref{bias_choice}). 
\end{theo}

\subsection{Differential Training for Nonlinear Classifiers}
\label{subsec:diff-nonlinear}

When a neural network is used to find a nonlinear classifier, a candidate cost function analogous to (\ref{brand_new_cost}) for differential training would be
\begin{equation} \sum\nolimits_{i \in I} \sum\nolimits_{j \in J} \log\left(1 + e^{-w^\top ( \phi_\theta(x_i) - \phi_\theta( y_j))} \right) \label{eqn:nonlinear-differential} \end{equation}
where $\phi_\theta(\cdot)$ is the output of the penultimate layer of the network and represents the features of the points. 
However, {\bf minimization of (\ref{eqn:nonlinear-differential}) has been observed to fail in providing a large margin in the input space} in our experiments. One reason for this is that the minimization of (\ref{eqn:nonlinear-differential}) does not guarantee a small Lipschitz constant for the mapping $\phi_\theta$. Therefore, even if the margin is large in the penultimate layer, the margin in the input space could still be very small. 

A \textbf{cost function that does provide a large margin} in the input space is
\begin{equation}  \sum\nolimits_{i\in I} \sum\nolimits_{j \in J} \left( w^\top \phi_\theta(x_i) - w^\top \phi_\theta(y_j) - 1 \right)^2. \label{squared-loss} \end{equation}
A partial explanation for the different behavior of this function is that the gradient descent algorithm is more likely to converge to a solution with small Lipschitz constant if the network is trained with the squared error loss~\cite{nipsStepSize}. Consequently, the gradient method is more likely to produce a $\phi_\theta$ which has a small Lipschitz constant, and this implies that the input of $\phi_\theta$ needs to change by a large amount in order for its output to move across the decision~boundary.


The effect of training with the cost function (\ref{squared-loss}) on the margin of a nonlinear classifier is demonstrated in Figure~\ref{fig:nonlinear-plot}.
A neural network with one hidden layer was trained with two different training cost functions: cross-entropy loss and the differential training cost (\ref{squared-loss}). The minimization of cross-entropy loss provided an extremely poor margin in the input space, whereas the use of (\ref{squared-loss}) lead to a decision boundary with large margins.

\begin{figure}
\centering
\includegraphics[width=0.75\linewidth, trim={0 4mm 0 0}, clip]{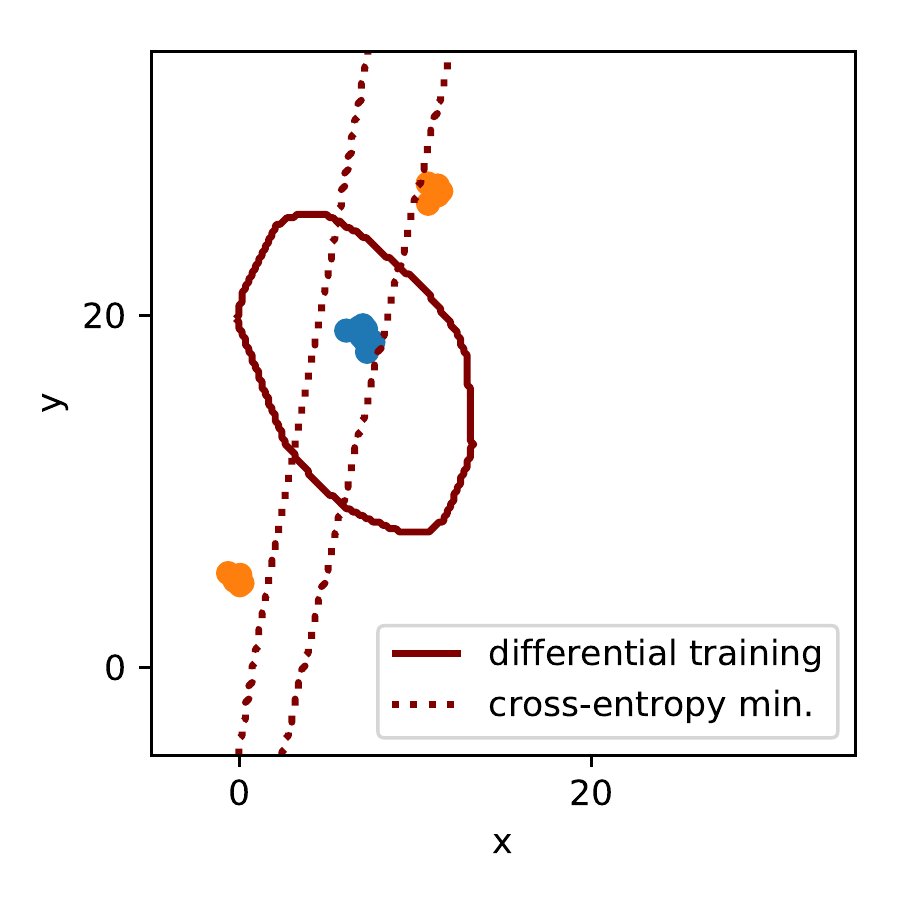}
\caption{A two-layer neural network is trained with two different cost functions. Cross-entropy minimization marks the region between the dotted lines as the class of blue points, whereas the same class is assigned to the region inside the solid curve when differential training is used. Note that the decision boundaries obtained with cross-entropy minimization have extremely small~margins.}
\vspace{-1mm}
\label{fig:nonlinear-plot}
\end{figure}

\section{Experiment on CIFAR-10: Differential Training Removes Adversarial Examples}
A large margin between the decision boundary of the classifier and the points in the training dataset is expected to make it harder to find adversarial examples for these points. 
In order to verify if this is the case, we trained a four-layer convolutional neural network for a binary classification task on CIFAR-10 dataset by only using the images for planes and horses.
Both cross-entropy minimization and differential training achieved zero error on the training dataset, and
the accuracies of both training schemes were comparable on the test dataset: cross-entropy minimization lead to 93.65\% while differential training yielded 94.65\%.

We generated adversarial examples for the images in the training dataset using Projected Gradient Descent Attack (PGD) implemented by~\cite{foolbox}.
The robustness of the neural network against these adversarial examples was substantially different based on whether the network was trained with the cross-entropy loss or the differential training cost (\ref{squared-loss}).

As shown in Figure \ref{fig:pgd-attack}, PGD was able to find adversarial examples for the images in the training dataset with small perturbations if the network was trained with the cross-entropy loss. In contrast, if the network was trained with differential training, PGD failed to find adversarial examples for the training dataset without disturbing the images by a large amount. Please note that PGD was considered to be the most powerful first-order gradient-based attack in~\cite{Madry}.

Somewhat surprisingly, the same behavior was observed on the test dataset as well. As displayed in Figure \ref{fig:pgd-attack}, PGD failed to find adversarial examples for most of the images in the test dataset when the network was trained via differential training. Moreover, the accuracy of the network was almost the same for adversarial examples generated from the training dataset and for those generated from the test~dataset.

\begin{figure}
\includegraphics[width=\linewidth]{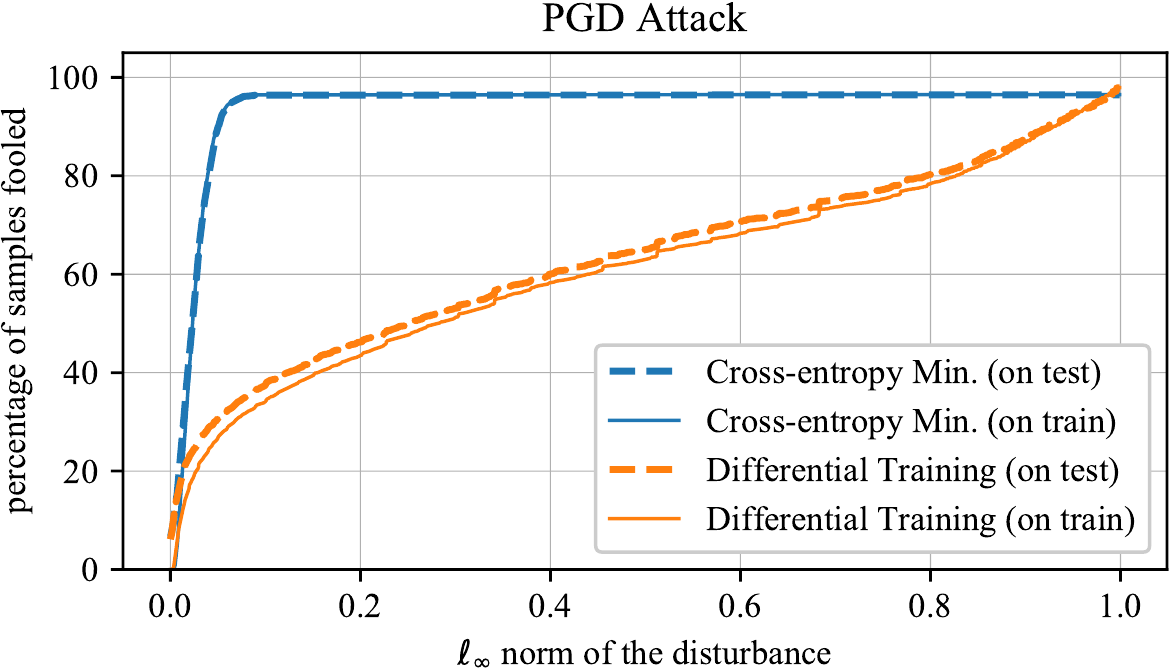}
\vspace{-4mm}
\caption{A four-layer convolutional neural network is trained for a binary classification task on CIFAR-10 dataset with two different training schemes: cross-entropy minimization and differential training.
If the network is trained with differential training, the accuracy of the network is much higher for the adversarial examples generated from the training and test datasets with the PGD Attack.
Moreover, the accuracy of the network on the adversarial examples generated from the training dataset is almost the same as its accuracy on those generated from the test dataset.
Solid lines denote the accuracy on adversarial examples generated from the training dataset, and dashed lines denote the accuracy on adversarial examples generated from the test dataset.
}
\label{fig:pgd-attack}
\end{figure}

We also tested the network under the Carlini-Wagner Attack~\cite{Carlini2017towards} implemented by~\cite{foolbox}. Similar to its performance under PGD Attack, the accuracy of the network trained with differential training remained much higher compared to the network trained with cross-entropy minimization, as shown in Figure~\ref{fig:carlini-wagner-attack}.

\begin{figure}
\includegraphics[width=\linewidth]{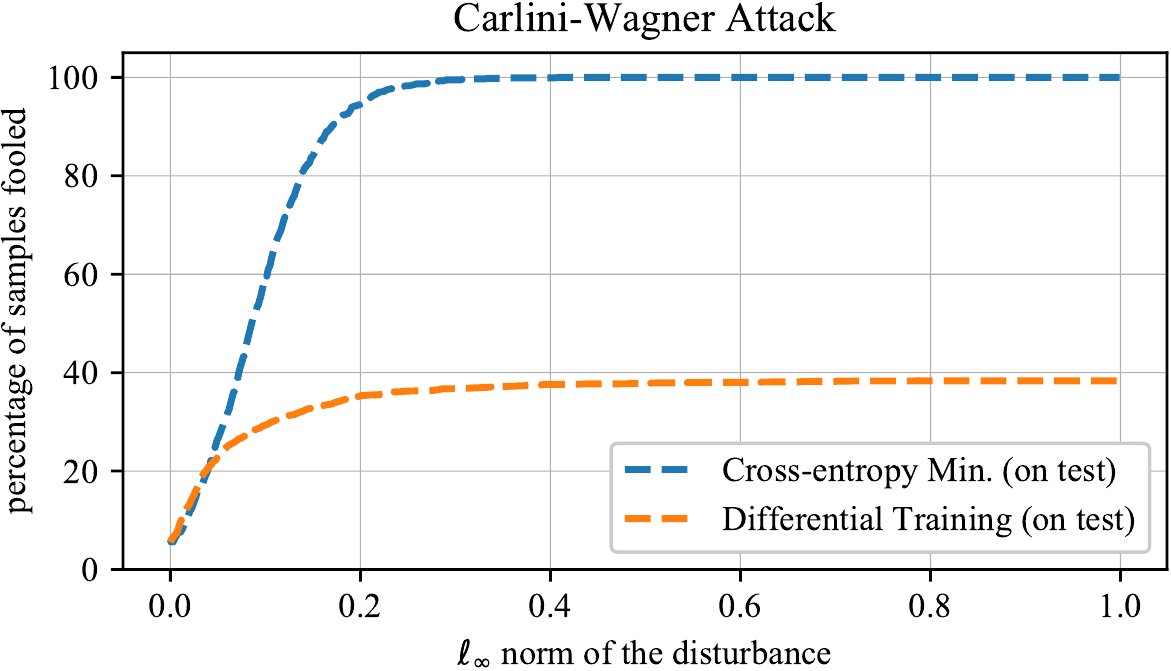}
\vspace{-5mm}
\caption{A four-layer convolutional network is trained with two different schemes: cross-entropy minimization and differential training. If the network is
trained with differential training,
the accuracy of the network 
is much higher on the adversarial examples generated from the test dataset with the Carlini-Wagner Attack.}
\label{fig:carlini-wagner-attack}
\end{figure}



\section{Discussion}

{\bf Low-dimensionality of the training dataset.} As stated in Remark 1, as the dimension of the affine subspace containing the training dataset gets very small compared to the dimension of the input space, the training algorithm will become more likely to yield a small margin for the classifier. This observation confirms the results of~\cite{Ramtin}, which showed that if the training dataset is projected onto a low-dimensional subspace before being fed into a neural network, the performance of the network against adversarial examples is improved -- since projecting the inputs onto a low-dimensional domain  corresponds to decreasing the dimension of the input space. Even though this method is effective, it requires the knowledge of the domain in which the training points are low-dimensional. Because this knowledge will not always be available a priori, finding alternative training algorithms and loss functions that are suited for low-dimensional data is still an important direction for future research.

{\bf Robust optimization.} 
Using robust optimization to train neural networks has been shown to be effective against adversarial examples~\cite{Madry, pmlr-v80-athalye18a}.
Note that these techniques could be considered as inflating the training points by a presumed amount and training the classifier with these inflated points. Nevertheless, as long as the cross-entropy loss is involved, the decision boundaries of the neural network will still be in the vicinity of the inflated points. Therefore, even though the classifier is robust against the disturbances of the presumed magnitude, the margin of the classifier could still be much smaller than what it could potentially be.

{\bf Differential training.} We introduced differential training, which allows the feature mapping to remain trainable while ensuring a large margin between different classes of points. By doing so, this method combines the benefits of neural networks with those of support vector machines. Even though moving from $2N$ training points to $N^2$ pairs might seem prohibitive, it points out that a true classification should in fact be able to differentiate between the pairs that are hardest to differentiate, and this search will necessarily require an $N^2$ term. Some heuristic methods are likely to be effective, such as considering only a smaller subset of points closer to the boundary and updating this set of points as needed during training. If a neural network is trained with this procedure, the network will be forced to find features that are able to tell apart between the hardest pairs.

{\bf Generalization of differential training, and its connection to one-shot learning.}
It has been shown that if a neural network is trained with robust optimization, the accuracy of the network on adversarial examples generated from the test dataset could be very low -- even though the accuracy on adversarial examples produced from the training dataset is high~\cite{schmidt2018}.
Consequently, it has been claimed that the robust optimization requires large amount of data so as to make a network robust against adversarial perturbations on the unseen images.
Our empirical results on CIFAR-10 dataset suggest that differential training does not suffer from this problem. That is, differential training provides neural networks with robustness while still using fewer data. This is in congruence with the main premise of~\cite{SiameseOneShot}, which showed that Siamese networks with an identical pair of networks in their architecture perform well with few training points. Please see Section \ref{sec:related-works} for further comments on the relation between differential training and Siamese networks.

{\bf Why not empirical risk minimization with a well-known loss function?} Consider the standard problem of empirical risk minimization as the proxy for finding a classifier:
\begin{equation}
 \min_{w, \theta} \sum\nolimits_{i \in I} \ell \left( w, \phi_\theta(x_i); z_i \right) \label{eq:erm} \end{equation}
where $z_i$ denotes the label of the point $x_i$, and $(w,\theta)$ are the parameters of the classifier. If the features of the training points $\{\phi_\theta(x_i)\}_{i \in I}$ lie in a low-dimensional subspace, the cost function (\ref{eq:erm}) will likely not be strictly convex; and more importantly, there will be directions in which the parameters are not penalized. 
Normally, the remedy would be to introduce a regularization term into the cost function. However, the effectiveness of well-known regularization terms is dubious for neural networks: they do not prevent spectral norms of weight matrices from growing unboundedly~\cite{BartlettSpectral}, nor do they influence the generalization gap of networks noticeably~\cite{RechtUnderstanding}.
Therefore, even if a regularization term is added externally, the gradient descent algorithm will have the potential to drive the parameters in the directions that are not penalized and cause the decision boundary to reside in the vicinity of the training points. Note that the loss function $\ell(\cdot)$ need not be the cross-entropy loss for this to happen.
This is why the problem of poor margins is in fact not peculiar to the cross-entropy loss,
and this is why other well-known loss functions will likely also fail in addressing adversarial examples.

\appendix

\section{Proof of Theorem 1 and Corollary 1}
\begin{lemma}[Adapted from Theorem 3 of \citep{Soudry-March-2018}] Given two sets of points $\{x_i\}_{i \in I}$ and $\{y_j\}_{j \in J}$ that are linearly separable in $\mathbb R^d$, let $\tilde x_i$ and $\tilde y_j$ denote $[ x_i^\top \ 1]^\top$ and $[y_j^\top \ 1]^\top$, respectively, for all $i \in I$, $j \in J$. Then the iterate of the gradient descent algorithm, $\tilde w(t)$, on the cross-entropy loss function
\begin{equation*}  \min_{\tilde w \in \mathbb R^{d+1}} \sum\nolimits_{i \in I} \log(1 + e^{-\tilde w^\top \tilde x_i}) + \sum\nolimits_{j \in J} \log(1 + e^{\tilde w^\top \tilde y_j}) \label{loss_cross_ent} \end{equation*}
 with a sufficiently small step size will converge in direction:
 \[ \lim_{t \to \infty} {\tilde w(t) \over \|\tilde w(t)\|} = {\overline w \over \|\overline w\|}, \] 
where $\overline w$ is the solution to 
\begin{align}
\label{eqn:svm-in-difference}
\underset{z \in \mathbb R^{d+1}}{\text{\emph{minimize}}}  &\quad  \|z\|^2  \\
\text{subject to} & \quad \langle z, \tilde x_i \rangle \ge 1 \quad \forall i \in I, \nonumber \\
&\quad  \langle z, \tilde y_j\rangle  \le 1 \quad  \forall j \in J. \nonumber \end{align}
\end{lemma}

{\bf Proof of Theorem 1.} 
Assume that $\overline w = u + \sum_{k=1}^m \alpha_k r_k$, where $u \in \mathbb R^{d}$ and $\langle u, r_k \rangle = 0$ for all $k \in K$. By denoting $z = [w^\top \ b]^\top$, 
the Lagrangian of the problem (\ref{eqn:svm-in-difference}) can be written as
\begin{gather*} {1\over 2} \| w\|^2 + {1 \over 2}b^2 + \sum\nolimits_{i\in I} \mu_i ( 1 - \langle  w,  x_i \rangle - b)\\ + \sum\nolimits_{j \in J} \nu_j (-1 + \langle  w,  y_j \rangle +b ), \end{gather*}
where $\mu_i \ge 0$ for all $i \in I$ and $\nu_j \ge 0$ for all $j \in J$.
KKT conditions for the optimality of $\overline w$ and $B$ requires that
\[ \overline w = \sum_{i\in I} \mu_i  x_i - \sum_{j \in J} \nu_j  y_j,\ \
 B = \sum_{i \in I} \mu_i - \sum_{j \in J} \nu_j,\]
and consequently, for each $k \in K$,
\begin{eqnarray*} \langle \overline w, r_k \rangle & = & \sum\nolimits_{i\in I} \mu_i \langle  x_i , r_k \rangle - \sum\nolimits_{j\in J} \nu_j \langle y_j, r_k \rangle \\ & = & \sum\nolimits_{i \in I} \Delta_k\mu_i - \sum\nolimits_{j \in J} \Delta_k \nu_j = B\Delta_k. \end{eqnarray*}
Then, we can write $\overline w$ as
\[ \overline w = u + \sum\nolimits_{k\in K} B\Delta_k r_k. \]
Let $ \langle w_\text{SVM}, \cdot \rangle + b_\text{SVM} = 0$ denote the hyperplane obtained as the solution of SVM. Then $w_\text{SVM}$ solves
%
%
%
%
\begin{align}  \underset{w}{\text{{minimize}}} & \quad \|  w \|^2  \label{eq:new-svm} \\
\text{subject to} & \quad \langle  w, x_i - y_j\rangle \ge 2 \quad \forall i \in I, \forall j \in J. \nonumber
\end{align}
Since the vector $u$ also satisfies $\langle u, x_i - y_j \rangle = \langle  w,  x_i -  y_j \rangle \ge 2$ for all $i \in I,j \in J$, we have $\|u\| \ge \|w_\text{SVM}\| = {1 \over \gamma}$. As a result, the margin obtained by minimizing the cross-entropy loss is 
\[ {1 \over \|\overline w\|} = { 1 \over \sqrt{ \|u\|^2 + \sum  \|B\Delta_k r_k\|^2}} \le {1 \over \sqrt{ {1 \over \gamma^2} + B^2\sum \Delta_k^2}}. \tag*{$\blacksquare$}\]

{\bf Proof of Corollary 1.}
If $B <0$, we could consider the hyperplane $\langle \overline w, \cdot \rangle - B = 0$ for the points $\{- x_i\}_{i\in I}$ and $\{-y_j\}_{j \in J}$, which would have the identical margin due to symmetry. Therefore, without loss of generality, assume $B \ge 0$. As in the proof of Theorem 1, KKT conditions for the optimality of $\overline w$ and $B$ requires
\[ \overline w = \sum_{i \in I} \mu_i  x_i - \sum_{j \in J} \nu_j  y_j , \ \ B = \sum_{i \in I} \mu_i - \sum_{j \in J} \nu_j \]
where $\mu_i \ge 0$ and $\nu_j \ge 0$ for all $i \in I, j \in J$. Note that for each $k \in K$,
\begin{eqnarray*}
 \langle \overline w, r_k \rangle & = & \sum\nolimits_{i \in I} \mu_i \langle  x_i, r_k\rangle - \sum\nolimits_{j \in J} \nu_j \langle y_j, r_k\rangle \\
 & = & B\Delta_k + \sum\nolimits_{i \in I} \mu_i (\langle x_i, r_k\rangle - \Delta_k ) \\
 & & - \sum\nolimits_{j \in J} \nu_j (\langle - y_j, r_k\rangle - \Delta_k ) \ \ge\  B\Delta_k.
 \end{eqnarray*}
 Since $\{r_k\}_{k\in K}$ is an orthonormal set of vectors,
 \[ \|\overline w\|^2 \ge \sum\nolimits_{k\in K} \left\langle \overline w, r_k \right\rangle^2 \ge \sum\nolimits_{k \in K} B^2 \Delta_k^2.\]
 The result follows from the fact that ${\|\overline w\|}^{-1}$ is an upper bound on the margin. \hfill $\blacksquare$

\section{Proposition 1 and Nonzero Initialization}
\label{appendix:nonzero-init}

Gradient descent algorithm on
\[ \sum\nolimits_{i \in I} \log(1 + e^{-w^\top W h_\theta(x_i)} ) \]
leads to the dynamics
\begin{equation}
 \dot W = w v^\top, \quad \dot w = W v,   \label{eqn:exog}
 \end{equation}
where
\[ v = \sum\nolimits_{i \in I} h_\theta(x_i) { e^{-w^\top Wh_\theta(x_i)} \over 1 + e^{-w^\top Wh_\theta(x_i)}}. \]
If $W(0)= 0$, then $w$ preserves its direction and $w(t) = w(0)\alpha(t)$ for all $t\ge 0$, where $\alpha(\cdot): [0,\infty) \to \mathbb R$. Consequently, the column space of $W(t)$ is spanned by only $w(0)$, and $W(t)$ has rank 1 or 0 for every $t\ge0$. This completes the proof of Proposition 1. In order to make a statement without the condition on $W(0)$, we need the following lemma.

\begin{lemma} \label{lemma:small} Consider the $n \times n$ matrix
\[
\left[ \begin{array}{l l} \mathbf{0} & v \\ v^\top & 0 \end{array} \right]
\]
where $v \in \mathbb R^{n-1}$ and assume $n \ge 2$. It has only one positive eigenvalue, $\|v\|_2$, with the eigenvector $[v^\top \ \|v\|_2]^\top$.
\end{lemma}

\begin{proof} The matrix is at most rank 2, so it has at most 2 nonzero eigenvalues. The vectors $[v^\top \ \|v\|_2]^\top$ and $[v^\top \ -\|v\|_2]^\top$ are its eigenvectors corresponding to the eigenvalues $\|v\|_2$ and $-\|v\|_2$, respectively. 
\end{proof}

In the dynamics (\ref{eqn:exog}), if we consider $v(t)$ as an exogenous signal, the system described becomes a linear time-varying system of the states $(W,w)$. Moreover, the dynamics of each row of the pair $(W,w)$ is independent of the other rows, but is governed by the same matrix. For example, the $k^\text{th}$ row of the pair $(W,w)$ satisfies:
\begin{equation}
\left[ \begin{array}{c} \dot W_{k1} \\ \vdots \\ \dot W_{kn} \\ \dot w_k \end{array} \right] = 
\left[ \begin{array}{c c} \mathbf{0} & v(t) \\ v(t)^\top & 0 \end{array} \right]
\left[ \begin{array}{c}  W_{k1} \\ \vdots \\  W_{kn} \\  w_k \end{array} \right]. 
\label{W-w-dynamics} \end{equation}

If the last layer of $h_\theta$ ends with a squishing function such as $\arctan$ or $\tanh$, and if all training points are classified correctly during training, the dynamics of $v$ becomes
\[ \dot v \simeq -\sum_{i \in I} h_\theta(x_i)e^{-w^\top W h_\theta(x_i)} (v^\top W^\top W + \|w\|^2v^\top)h_\theta(x_i)  \]
if the network is trained for long enough.
Then the change in $v$ becomes exponentially slower than those in $W$ and $w$ as the training continues. Consequently, the vector $v(t)$ in (\ref{W-w-dynamics}) acts as a constant vector; and from Lemma \ref{lemma:small}, each row of the matrix $W$ grows in the direction $v(t)$ by the same ratio. As a result, if the algorithm is run for long, all rows of $W$ converge to the same direction. Correspondingly, all of its columns converge to a set with rank 1 (or 0).

\section{Proof of Theorem 2}
Apply Lemma 1 by replacing the sets $\{x_i\}_{i \in I}$ and $\{y_j\}_{j \in J}$ with $\{x_i - y_j\}_{i \in I, j \in J}$ and the empty set, respectively. Then the minimization of the loss function (\ref{brand_new_cost}) with the gradient descent algorithm leads to
\[ \lim_{t \to \infty} {w \over \|w\|}  = {\overline w \over \| \overline w \|} \]
where $\overline w$ satisfies
\[ \overline w = \arg\min_{w } \|w\|^2 \ \text{ s.t. } \ \langle w, x_i - y_j \rangle \ge 1 \ \forall i \in I, \ \forall j \in J. \]
Since $w_\text{SVM}$ is the solution of (\ref{eq:new-svm}),
we obtain $\overline w = {1 \over 2} w_\text{SVM}$, and the claim of the theorem holds. \hfill $\blacksquare$

\section{Proof of Theorem 3}
In order to achieve zero training error in one iteration of the stochastic gradient algorithm, it is sufficient to have
\[ \min_{i' \in I} \langle x_{i'}, x_i - y_j \rangle > \max_{j' \in J} \langle y_{j'}, x_i - y_j \rangle \ \forall i \in I, \ \forall j \in J, \]
or equivalently,
\begin{equation} \langle x_{i'} - y_{j'} , x_i - y_j \rangle > 0 \quad \forall i, i' \in I, \ \forall j, j' \in J. \label{differdiffer} \end{equation}
By definition of the margin, there exists a vector $w_\text{SVM} \in \mathbb R^d$ with unit norm which satisfies
\[ 2\gamma = \min\nolimits_{i \in I, j \in J} \langle x_i - y_j , w_\text{SVM} \rangle. \]
Note that $w_\text{SVM}$ is orthogonal to the decision boundary given by the SVM. Then we can write every $x_i - y_j$ as
\begin{gather*} x_i - y_j = 2\gamma w_\text{SVM} + \delta^x_i + \delta^y_j, 
\end{gather*}
where $\delta^x_i,  \delta^y_j \in \mathbb R^d$ and $\|\delta^x_i\| \le R_x$ and $\|\delta^y_j\| \le R_y$. 
Then, condition (\ref{differdiffer}) is satisfied if
\[ \langle  2\gamma w_\text{SVM} + \delta^x_i + \delta^y_j , 2\gamma w_\text{SVM} + \delta^x_{i'} + \delta^y_{j'} \rangle > 0 \]
for all $i, i' \in I$ and for all $j, j' \in J$; or equivalently if
\begin{equation}
4\gamma^2 + 2\gamma \langle w_\text{SVM} , \delta^x_i + \delta^y_j + \delta^x_{i'} + \delta^y_{j'} \rangle + \langle \delta^x_i + \delta^y_j , \delta^x_{i'} + \delta^y_{j'} \rangle > 0 \label{unnecessary_eqn} \end{equation}
for all $i, i' \in I$ and for all $j, j' \in J$.
If we choose $\gamma > {5 \over 2} \max(R_x, R_y)$, we have
\[ 4\gamma^2 - 2\gamma(2R_x + 2R_y) - (R_x+R_y)^2 > 0, \]
which guarantees (\ref{unnecessary_eqn}) and completes the proof. \hfill $\blacksquare$

\bibliography{draft}
\bibliographystyle{icml2019}

\end{document}